\documentclass[letterpaper, 10 pt, conference]{ieeeconf}  

\IEEEoverridecommandlockouts                              

\usepackage[english]{babel}
\usepackage[utf8]{inputenc}
\usepackage[cmex10]{amsmath}
\usepackage{amssymb}

\usepackage{float}
\usepackage{graphicx}
\usepackage{amsfonts}
\usepackage{mathtools}
\usepackage{xcolor}
\usepackage{color}
\usepackage{verbatim}
\usepackage{graphicx}
\usepackage{subcaption}
\usepackage{soul}

\newtheorem{definition}{Definition}{}
{}
\newtheorem{proposition}{Proposition}{}
\newtheorem{theorem}{Theorem}{}
\newtheorem{remark}{Remark}{}
\newtheorem{lemma}{Lemma}{}

\title{Performance Bounds for Neural Network Estimators: Applications in Fault Detection}

\author{Navid Hashemi, Mahyar Fazlyab, Justin Ruths
\thanks{N. Hashemi and J. Ruths are with the Department of Mechanical Engineering, The University of Texas at Dallas. Email: {\tt\small (nxh150030, jruths)@utdallas.edu}. M. Fazlyab is with the John Hopkins Mathematical Institute for Data Sciences. Email: {\tt\small mahyarfazlyab@jhu.edu}}
}

\begin{document}

\maketitle
\thispagestyle{empty}
\pagestyle{empty}

\begin{abstract}
We exploit recent results in quantifying the robustness of neural networks to input variations to construct and tune a model-based anomaly detector, where the data-driven estimator model is provided by an autoregressive neural network. In tuning, we specifically provide upper bounds on the rate of false alarms expected under normal operation. To accomplish this, we provide a theory extension to allow for the propagation of multiple confidence ellipsoids through a neural network. The ellipsoid that bounds the output of the neural network under the input variation informs the sensitivity - and thus the threshold tuning - of the detector. We demonstrate this approach on a linear and nonlinear dynamical system.
\end{abstract}

\section{Introduction}
The rise in interest in data-driven techniques is a response to the need for better models of complex systems. Model-based fault and anomaly detection in dynamical systems typically employs first-principle models (Newton's and Kirchhoff's Laws, reaction kinetics, etc.) of systems to identify discrepancies between the observed and predicted sensor measurements. However, to use this general approach on large-scale operations, such as petro-chemical refineries and power distribution grids, operators need scalable and efficient techniques for creating models. Data-driven tools offer a compelling option because models can be devised simply from past data, do not require intimate knowledge of system parameters, and can be recomputed regularly to update the dynamics that may change over time.  

Artificial neural networks, in particular, provide relatively easy training and generalization, simple architecture, good ability to approximate nonlinear functions, and robust to inexact input data \cite{hassoun1995fundamentals}. 
Neural networks can be used to identify and control nonlinear dynamic systems because they can approximate a wide range of nonlinear functions. 
Over the past two decades, fault detection has employed neural networks in a variety of ways, including using them as classifiers to categorize normal or various faulty behaviors \cite{liu1999extended} and using them as a model of the system to provide an estimate for subsequent fault detection \cite{vemuri1997neural,wlas2005artificial}. In the latter case, neural network estimators can be implemented using an iterative approach (in which the past estimate(s) is part of the input to produce the next estimate) or an autoregressive approach (in which a finite history of past measurements is used to produce the next estimate) \cite{frank1997new}\nocite{li2005fault}-\cite{abbaspour2017neural}. The detection problem becomes particularly challenging when the fault is unknown or cannot be modeled (e.g., we want to build a detector that is sensitive to potentially unknown faults or anomalies, like attacks). In this context, tuning of the detector is important to balance the sensitivity of the detector with the prevalence of false alarms and while data-driven techniques are not new to fault detection, tuning to-date has been done empirically. In this paper we leverage recent results to quantify the robustness of neural networks \cite{fazlyab2019probabilistic,fazlyab2019safety} to develop a novel data-driven anomaly detector with guarantees on the upper limit of the false alarm rate.


Quantifying the robustness of neural networks is originally motivated by their vulnerability to adversarial attacks, i.e., carefully chosen small input perturbations that can drastically change their output. To this end, a plethora of tools have been developed to bound the output of neural networks for a given range of inputs (e.g., the set of plausible attacks) ~\cite{fazlyab2019safety}\nocite{dutta2017output, lomuscio2017approach, wong2018provable, huang2017safety, raghunathan2018semidefinite,katz2017reluplex}-\cite{hashemi2020certifying}. Our particular interest in this paper is the uncertainty propagation technique put forth in \cite{fazlyab2019probabilistic}, in which the authors develop a semi-definite program that propagates an input ellipsoid (e.g., confidence region of a density function) through the neural network to obtain \emph{guaranteed} ellipsoidal over-approximation of the output set. In another context, Monte Carlo sampling, the Unscented Transform (UT) and Extended Kalman Filtering (EKF) have been used to take a set of samples from the input distribution, propagate them through the neural network, and approximate the first and second moments of the output distribution from them \cite{abdelaziz2015uncertainty,titensky2018uncertainty}. These sampling methods scale to larger neural networks but they lack formal guarantees.

\section{Background}
Consider a discrete-time dynamical system whose state update can be described by an \textit{unknown} continuous function, $\mathcal{F}:\mathbb{R}^n\to\mathbb{R}^n$,
\begin{equation}
    x_{k+1} = \mathcal{F}(x_k),
\end{equation}
which may arise as a freely evolving dynamical system or as a feedback control system. Our observations (sensor measurements) of the system are linear combinations of the states corrupted by additive zero-mean noise $v_k\in\mathbb{R}^p$, with known covariance $\Sigma_v$,
\begin{equation} \label{eq:observation}
    y_k=Hx_k+v_k,
\end{equation}
where the sensor matrix $H\in\mathbb{R}^{p\times n}$ is known. Because the system model is unknown, or possibly too complicated with too many parameters to use in an effective way, we use a data-driven estimator with the assumption that it is possible to construct an accurate nonlinear auto-regressive (NARX) model of the output observations based on past measurement values (i.e., the system is fully observable),
\begin{equation}\label{eq:ffNN}
    y_{k+1}=f(\mathcal{Y}_{k,N}),
\end{equation}
where $\mathcal{Y}_{k,N}=\textbf{vec}\left[\left\{ y_{k-i} \right\}_{i=0}^N \right]\in\mathbb{R}^{p(N+1)}$. 

Again, the mapping $f: \mathbb{R}^{n\times p}\to \mathbb{R}^p$ is \textit{unknown} or otherwise too complicated to form from first-principle models, so we approximate this mapping through supervised training of a feedforward neural network from time-series recordings of sensor measurements. At each time $k\geq N$ the network is provided  the vector of current and $N$ past measurements $\mathcal{Y}_{k,N}$ as input training data and the next measurement $y_{k+1}$ as the labeled output data of the neural network. The network is trained under a large number of varying initial conditions.

The $\ell$-layer neural network, trained with input data $\mathcal{Y}_{k,N}$ and labeled data $y_{k+1}$, is
\begin{equation}\label{observer}
    \begin{aligned}
    \zeta_k&=\left[\ \mathcal{Y}_{k,N}\ \right]^\top  ,\\
    z^0&=\zeta_k,\\
    z^{t+1}&=\phi(W^tz^t+b^t),\quad t=0,1, \dots, \ell-1,\\
    \hat{y}_{k+1}&=W^{\ell}z^\ell +b^\ell,
    \end{aligned}
\end{equation}
with ReLU activation function $\psi(s_i)=\max{(0,s_i)}$, $\forall s_i \in \mathbb{R}$, and $\phi([s_1,s_2,..,s_d])=[\psi(s_1),\psi(s_2), ..., \psi(s_d)]$. We train the model to compute optimal weight matrices and bias vectors and we call the output of the neural network, $\hat{y}_{k+1}$, as the estimated/predicted output.

The primary role of the training process is to refine the neural network to produce \textit{accurate} estimates, i.e., predictions $\hat{y}_k$ that are close to the actual measurements $y_k$. The purpose of this paper is not to refine the training process to improve accuracy, but to quantify the robustness of the prediction to input uncertainties for a \textit{given} trained neural network (e.g., uncertainty caused by sensor noise). Further, we aim to use the robustness bounds as a way to quantify normal behavior from abnormal behavior.  

\begin{remark}
The NARX architecture is one of several ways to construct a neural network estimator of a dynamical system. An alternative choice is to feed the past estimation $\hat{y}_k$ as an input to the prediction for $\hat{y}_{k+1}$. The autoregressive approach we use here simplifies the training process and will also simplify the robustness quantification since the bounds on the neural network inputs are constant across time.
\end{remark}


\begin{figure*}[t]
    \centering
    \includegraphics[width=0.75\linewidth]{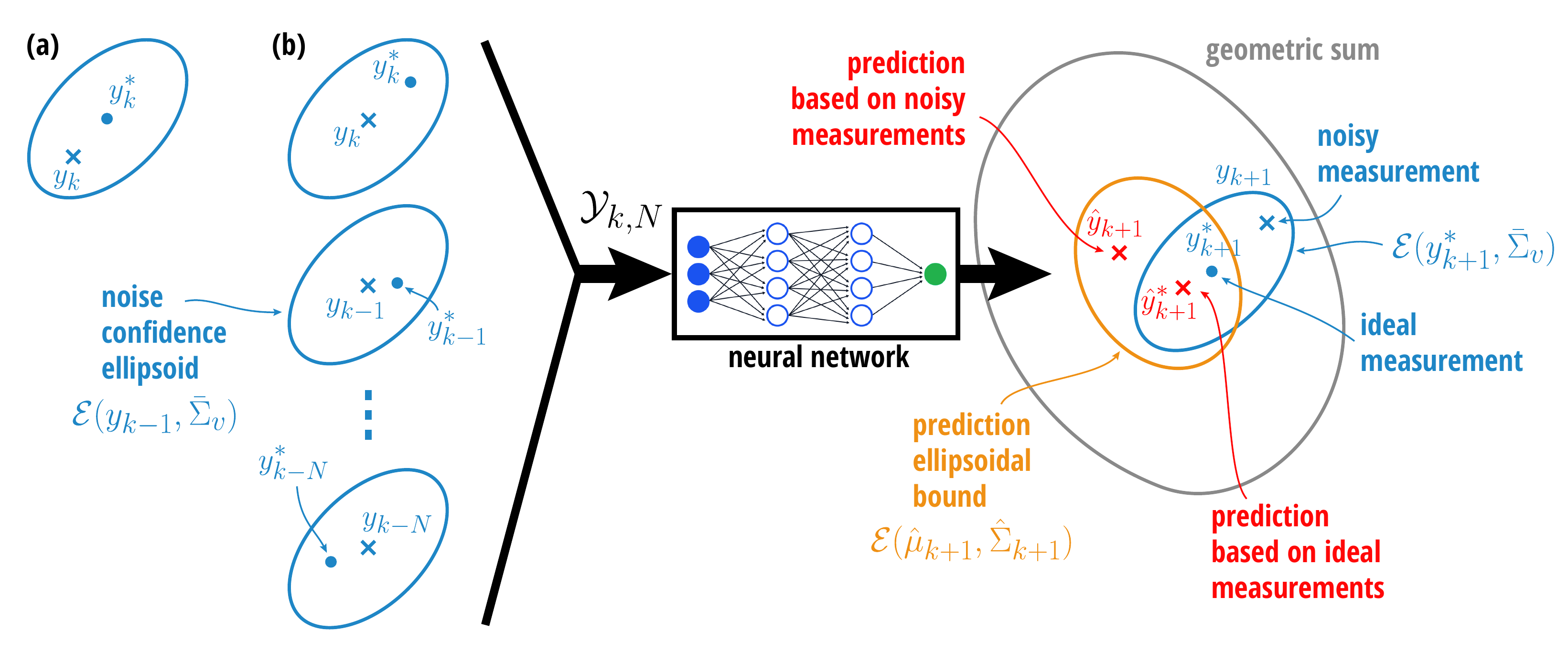}
    \caption{The proposed detector relies on the development of a prediction ellipsoid (orange) $\mathcal{E}(\hat{\mu}_{k+1},\hat{\Sigma}_{k+1})$ that bounds the variation of the estimate due to the noise inherent in the input measurements $y_k$, $y_{k-1}$, \dots, $y_{k-N}$ and captured by the input ellipsoids $\mathcal{E}(y_{k-i},\bar{\Sigma}_v)$, $i=0,\dots,N$. The framework to compute the prediction ellipsoid is presented in Section \ref{sec:robustness} and the detector that uses this ellipsoid to decide between normal and anomalous behavior is presented in Section \ref{sec:detector}.} 
    \label{fig:idea}
\end{figure*}

\section{Neural Network-Based Anomaly Detector} \label{sec:detector}
In the same spirit as propagating the estimation covariance of a Luenberger observer (e.g., Kalman Filter), evaluating the robustness of a neural network estimator identifies the quality of the estimation - quantifying how much the estimator can reduce the influence of the noise on the prediction. At the same time, quantifying the inherent variation in the predictions due to uncertainty (noise) also provides a bound on the variation observed during normal behavior. 

Consider the noisy measurement $y_k$ from \eqref{eq:observation} at a particular time step $k$, which is composed of the non-noisy (deterministic) ``ideal'' measurement,
\begin{equation}
    y^*_k=H x_k,
\end{equation}
and a sample of the sensor noise, $v_k$. If a different realization of the sensor noise $v_k$ was instead used, our goal is to bound how different the prediction $\hat{y}_{k+1}$ would be under these two sensor realization scenarios. Said differently, we would like to assess the robustness of the prediction to the inherent perturbations caused by sensor noise.

Because the support of the noise could be large, or possibly infinite (e.g., Gaussian noise), it is practically useful to truncate the noise distribution at a desired confidence set, such that the probability $\bar{p}$ that a noise sample is drawn from within the confidence set is a desired value (typically close to one). An ellipsoidal confidence set can be constructed using a scaled version of the covariance of the noise distribution as the shape matrix and characterized by
\begin{equation}
    \Pr[v_k^\top \Sigma_v^{-1}v_k \leq \alpha] = \Pr[v_k^\top (\underbrace{\alpha\Sigma_v}_{\bar{\Sigma}_v})^{-1}v_k \leq 1] =\bar{p}.
\end{equation}
This defines the $\bar{p}$-confidence ellipsoid on the sensor noise, $\mathcal{E}(0,\bar{\Sigma}_v)$, with an ellipsoid with center $\mu$ and shape matrix $\Sigma$ defined as
\begin{equation}
    \mathcal{E}(\mu,\Sigma) := \left\{ \xi\ \big|\ (\xi-\mu)^\top \Sigma^{-1}(\xi-\mu) \leq 1 \right\},
\end{equation}
and the size of the ellipsoid to match the desired confidence level $\bar{p}$ is chosen by selecting $\alpha=2\Gamma^{-1}(\frac{p}{2},\bar{p})$, using the inverse regularized lower incomplete gamma function and $p$ is the number of measurements at each time step \cite{murguia2019IET}.  

One way to interpret this bound is that since $y_k=y^*_k + v_k$, the actual measurement $y_k$ is contained within an ellipsoid $\mathcal{E}(y_k^*,\bar{\Sigma}_v)$, i.e., centered at the ideal measurement $y^*_k$ (see Fig. \ref{fig:idea}a). However, we can also see that $y^*_k = y_k-v_k$ which says the opposite - that the ideal measurement is contained within the ellipsoid centered at the actual measurement, $y^*_k\in\mathcal{E}(y_k,\bar{\Sigma}_v)$ (see Fig. \ref{fig:idea}b). The latter interpretation is practically useful because we have access to the actual measurements, but not the ideal (noise-free) measurements. 

The primary contribution of this paper is to construct and tune a model-based anomaly detector that employs a neural network to produce an estimate. Using Fig. \ref{fig:idea} as a guide, this objective can be decomposed into two pieces
\begin{enumerate}
    \item the framework to quantify the robustness of the neural network under input perturbations ($N$ blue input confidence ellipsoids), $\mathcal{E}(y_{k-N},\bar{\Sigma}_v), \cdots, \mathcal{E}(y_{k},\bar{\Sigma}_v)$ to produce an ellipsoidal bound on the predictions (orange ellipsoid) which we will denote by $\mathcal{E}(\hat{\mu}_{k+1},\hat{\Sigma}_{k+1})$; and 
    \item the definition of a detector that uses this (orange) ellipsoid to bound normal behavior through the use of the geometric sum.
\end{enumerate}
We address the detector definition first in the following subsection before presenting the mathematical framework to create the prediction bound in Section \ref{sec:robustness}.


\subsection{Detector Definition}
Given the ellipsoidal confidence sets that quantify the perturbation of the inputs to the neural network estimator, in Section \ref{sec:robustness} we will compute an ellipsoidal bound on the output of the estimate, which we call the prediction bound $\mathcal{E}(\hat{\mu}_{k+1},\hat{\Sigma}_{k+1})$. This bound guarantees that any $N+1$ past measurement vectors which stay within their respective confidence sets are mapped by the neural network to a point within the prediction bound ellipsoid. 

From Fig. \ref{fig:idea}, since the ideal past measurements $y_{k-N}^*$, \dots, $y_k^*$ are within the input confidence ellipsoids, we know the neural network output estimate $\hat{y}_{k+1}^*$ will be located within the prediction bound. The accuracy of the estimator can be interpreted as the distance between this prediction based on ideal measurements $\hat{y}_{k+1}^*$ and the next ideal measurement $y_{k+1}^*$. In our framework presented here, we will assume that the trained model is accurate such that this distance is relatively small and can be neglected; however, an ellipsoidal bound on the estimation accuracy could be easily integrated into our detector definition. 

As discussed above, the actual noisy measurement at $k+1$ is, with probability $\bar{p}$ within the confidence ellipsoid centered at the ideal measurement, $y_{k+1}\in\mathcal{E}(y^*_{k+1},\bar{\Sigma}_v)$. Although we do not know the ideal measurement, we know it is within the prediction bound. This leads us to suggest a detector that evaluates normal behavior as measurements based on the geometric sum between the prediction bound and the confidence ellipsoid.

\begin{definition} \label{def:detector}
A model-based anomaly detector that employs an autoregressive neural network to produce an estimate from past measurements raises alarms using the following logic:
\begin{equation}
\left\{\ 
\begin{aligned}
    y_{k+1} \in \mathcal{E}(\hat{\mu}_{k+1},\hat{\Sigma}_{k+1}) \oplus \mathcal{E}(0,\bar{\Sigma}_v) \ &\to\ \  \text{no alarm}, \\
    \text{otherwise}\hspace{3.7cm} &\to\ \ \text{alarm},
\end{aligned}
\right.
\end{equation}
where the prediction bound is characterized by a center $\hat{\mu}_{k+1}$ and shape matrix $\hat{\Sigma}_{k+1}$; the noise $\bar{p}$-confidence ellipsoid is characterized by the shape matrix $\bar{\Sigma}_v$; and $\oplus$ denotes the geometric (Minkowski) sum of two sets $\mathcal{S}_1\oplus\mathcal{S}_2 = \{ s_1+s_2\ |\ s_1\in\mathcal{S}_1,\ s_2\in\mathcal{S}_2\}$.
\end{definition}

\begin{remark}
Since the geometric sum of two ellipsoids is, in general, not an ellipsoid, it is best to verify this inclusion directly rather than to first approximate the sum as an ellipsoid. There are a number of computational techniques to verify the inclusion of a point in the geometric sum of ellipsoidal sets such as linear matrix inequalities \cite{boyd2004convex} or geometric methods to compute the exact geometric sum \cite{hashemi2020co}.  
\end{remark}

The purpose of a detector is to alert the operator of a system to potential anomalies, such as faults or attacks. The number of \textit{false alarms}, i.e., alarms raised during normal operation, is a key way to assess the  performance of a detector. Tuning the detector to balance sensitivity and false alarms is a key step towards making a detector usable in practice.

\begin{proposition}
The detector defined in Definition \ref{def:detector} has a false alarm rate that is upper bounded by $1-\bar{p}^{N+2}$.
\end{proposition}

\begin{proof}
The fact that the sensor noise is independent allows us to quantify probabilities easily. The confidence ellipsoids are used $N+1$ times for times $k-N$, \dots, $k-1$, $k$ (the input ellipsoids) and once for time $k+1$ in the geometric sum. Thus since $v_k$ is independent, the probability that all these $v_{k-N}$, \dots, $v_{k+1}$ fall within their confidence sets is $\bar{p}^{N+2}$. Because the prediction ellipsoid is an outer bound and because there could be noise realizations that lie outside their confidence ellipsoids but still remain inside the prediction bound, this probability is a lower bound. Thus under normal operation, we expect the probability of generating a false alarm to be at most $1-\bar{p}^{N+2}$.
\end{proof}

\section{Robustness of Prediction} \label{sec:robustness}

\begin{figure*}[t]
\small
\begin{align}
E_i=\left[\begin{matrix}  
    \begin{array}{c|c|c}  
     0_{n_i \times \left(\sum_{j=1}^{i-1}n_j\right)} \ \  I_{n_i} \ \   0_{n_i \times \left(\sum_{j=i+1}^{q}n_j\right) } &  0_{n_i \times \left(\sum_{t=2}^{\ell+1}N_j\right)}  & 0_{n_i \times 1}\\
    \hline
     0_{1 \times n_{\Gamma}} &  0_{1 \times \left(\sum_{t=2}^{\ell+1}N_j\right)}  & 1 \end{array}  
     \end{matrix}\right]
 \label{eq:selectionmatrix}\tag{$\star$} 
\end{align}
\normalsize
\hrule
\end{figure*}

Our detector design leverages the ability to compute a bound on the predictions made by the neural network estimator under perturbations to the input. To accomplish this bound, we use the formulation described in \cite{fazlyab2019probabilistic}, which provides an ellipsoidal bound on the output of a neural network given a single ellipsoidally bounded input vector. Here, as seen in Fig. \ref{fig:idea}, our approach uses multiple ellipsoidally bounded input vectors and, therefore, requires an extension to the method in \cite{fazlyab2019probabilistic}. The fundamental difference in these two approaches in the context of our problem is at what stage the confidence level truncation is applied. In our proposed multi-ellipsoid input approach we construct a confidence ellipsoid for each individual measurement. To use the single ellipsoid input approach of \cite{fazlyab2019probabilistic} directly, we would define an overall confidence ellipsoidal set on the entire stacked input vector. The multiple ellipsoid approach has a few key benefits: (1) while the noise is independent, the center of the input ellipsoids $y_{k-N}$, \dots, $y_k$ are correlated due to the deterministic dynamics of the system, which complicates concatenating the measurements; (2) the input ellipsoids all have the same shape matrix whereas the single input ellipsoid shape matrix would be time dependent; and (3) 
in Appendix \ref{appendixA} we show that the multiple ellipsoid input approach results in a less conservative ellipsoidal prediction bound. 

Here we follow the general approach in \cite{fazlyab2019probabilistic,fazlyab2019safety}, introducing the updates needed for multiple ellipsoidally bounded inputs. Consider the following general neural network with $\ell$ hidden layers,
    \begin{equation}\label{eq:NN}
    \begin{aligned}
    \Gamma_k&=[\gamma_{1,k}^\top ,\  \ \gamma_{2,k}^\top ,\ ..., {\gamma_{q,k}}^\top ]^\top \\
    z^0&=\Gamma_k\\
    z^{t+1}&=\phi(W^t z^t+b^t),\quad t=0,1,...,\ell-1\\
    \pi(z^0)&=W^{\ell}z^\ell +b^\ell,
    \end{aligned}
    \end{equation}
with $\gamma_{i,k} \in \mathbb{R}^{n_i}$, $\Gamma_k \in \mathbb{R}^{n_{\Gamma}}$ ($\sum_{j=1}^q n_i=n_\Gamma$), $\pi(z^0) \in \mathbb{R}^{n_{\pi}}$ and $z^t \in \mathbb{R}^{N_t}$. 
By concatenating all the post-activation values as $\mathbf{z}=\left[\ {z^{0}}^\top,\ {z^{1}}^\top, ..., {z^{\ell}}^\top \right]^\top  \in \mathbb{R}^{N_\mathbf{z}}$ ($N_{\mathbf{z}} = \sum_{t=0}^\ell N_t$), and defining the ``selector matrices'' $S^t$ such that $z^t = S^t\mathbf{z}$, we can rewrite the network as
    \begin{equation}\label{eq:NN2}
    \begin{aligned}
    \Gamma_k&=[\gamma_{1,k}^\top ,\  \ \gamma_{2,k}^\top ,\ ..., {\gamma_{q,k}}^\top ]^\top \\
    z^0 &= S^0\mathbf{z} =\Gamma_k\\
    B\mathbf{z}&= \phi(A \mathbf{z}+b)\\
    \pi(z^0)&=W^{\ell}S^{\ell}\mathbf{z} +b^{\ell},
    \end{aligned}
    \end{equation}
    where
    \begin{align}
    &A=\begin{bmatrix}W^0 & 0 & \dots &0 & 0 \\ 0 & W^1 & \dots & 0 & 0\\ \vdots & \vdots & \ddots & \vdots & \vdots \\ 0 & 0 & \dots & W^{\ell-1} & 0\end{bmatrix}, \quad b=\begin{bmatrix}b^0\\b^1\\ \vdots \\ b^{\ell-1} \end{bmatrix},\nonumber\\
    & B=\begin{bmatrix}0 & I_{n_1} & \dots & 0 \\ \vdots & \vdots & \ddots & \vdots  \\ 0& 0 & \dots & I_{n_{\ell}}\end{bmatrix}.
    \end{align}
Suppose each input $\gamma_{i,k}$ takes values inside the ellipsoid $\mathcal{E}(\mu_{\gamma_{i,k}},\Sigma_{i,k})$. Our goal is to bound the resulting output $\pi(z^0)$ of the neural network by an ellipsoid. To this end, we first abstract the neural network via Quadratic Constraints (QC) \cite{fazlyab2019safety} and then use the S-procedure to propagate the $q$ input ellipsoids through the network. We begin with the following definition.
%

\begin{definition}[QC for functions \cite{9301422}]\label{QQC}
Let $\phi: \mathbb{R}^d \to \mathbb{R}^d$ and suppose $\mathcal{Q}_{\phi} \subset \mathbb{S}^{2d+1}$  is the set of all symmetric indefinite matrices $Q$ such that
     \begin{equation}
        \begin{bmatrix}x\\\phi(x)\\1\end{bmatrix}^\top  Q\begin{bmatrix}x\\\phi(x)\\1\end{bmatrix}\geq 0, \quad \forall \ x \in \mathcal{X},
    \end{equation} 
where $\mathcal{X} \subset \mathbb{R}^d$ is a nonempty set. Then we say $\phi$ satisfies the QC defined by $\mathcal{Q}_\phi$ on $\mathcal{X}$.
\end{definition}
%
\begin{lemma}\label{lem:INOUTNN}
Consider the neural network described in \eqref{eq:NN2}. 
\begin{enumerate}[leftmargin=*]
\item Suppose $\gamma_{i,k}\in \mathcal{E}(\mu_{\gamma_{i,k}}, \Sigma_{i,k})$. Then for any $\tau_i \geq 0$ we have
\begin{equation}\label{eq:InputQC}
    \begin{bmatrix} \mathbf{z} \\1 \end{bmatrix}^\top  \left( \sum_{i=1}^{q} \tau_i M_i \right) \begin{bmatrix} \mathbf{z} \\1 \end{bmatrix}\geq 0,
\end{equation}
where
\begin{equation*} 
    M_{i}\!=\!E_i^\top  \begin{bmatrix}-\Sigma_{\gamma_{i,k}}^{-1} & \Sigma_{\gamma_{i,k}}^{-1}\mu_{\gamma_{i,k}}\\ \mu_{\gamma_{i,k}}^\top \Sigma_{\gamma_{i,k}}^{-1} & -\mu_{\gamma_{i,k}}^\top  \Sigma_{\gamma_{i,k}}^{-1}\mu_{\gamma_{i,k}}+1  \end{bmatrix} E_i,
\end{equation*}
and $E_i$ is defined as in \eqref{eq:selectionmatrix}.
\item Let $U\in \mathbb{S}^{n_{\pi}}$, $V \in \mathbb{R}^{n_{\pi}}$. Suppose $\mathbf{z}$ satisfies the the quadratic inequality
\begin{equation}\label{eq:OutputQC}
    \begin{bmatrix}\mathbf{z} \\ 1 \end{bmatrix}^\top M_{out}\begin{bmatrix}\mathbf{z} \\ 1 \end{bmatrix} \leq 0,
\end{equation}
where
\begin{equation*}
    M_{out} \!=\! \begin{bmatrix}W^{\ell}S^{\ell} & b^\ell\\ 0 & 1\end{bmatrix}^\top  \begin{bmatrix}U^2 & UV \\ V^\top  U &V^\top  V\!-\!1\end{bmatrix} \begin{bmatrix}W^{\ell}S^{\ell} & b^\ell\\ 0 & 1\end{bmatrix}.
\end{equation*}
Then we have $\pi(z^0) \in \mathcal{E}(-U^{-1}V, U^{-2})$.
\end{enumerate}
\end{lemma} 
\vspace{1em}
\begin{proof}
(1) The proof is a slight modification of \cite{fazlyab2019probabilistic}. We know that each input $\gamma_{i,k}$ is bounded by the ellipsoid $\mathcal{E}(\mu_{\gamma_{i,k}}, \Sigma_{\gamma_{i,k}})$ , which means \begin{equation} \label{eq:input_ellipsoidal_bound}
    (\gamma_{i,k}-\mu_{\gamma_{i,k}})^\top \Sigma_{\gamma_{i,k}}^{-1}(\gamma_{i,k}-\mu_{\gamma_{i,k}})<1.
    \end{equation}
This can be rewritten as
    \begin{equation*}
    \begin{bmatrix} \gamma_{i,k} \\1 \end{bmatrix}^\top  \begin{bmatrix}-\Sigma_{\gamma_{i,k}}^{-1} & \Sigma_{\gamma_{i,k}}^{-1}\mu_{\gamma_{i,k}}\\ \mu_{\gamma_{i,k}}^\top \Sigma_{\gamma_{i,k}}^{-1} & -\mu_{\gamma_{i,k}}^\top  \Sigma_{\gamma_{i,k}}^{-1}\mu_{\gamma_{i,k}}+1  \end{bmatrix}\begin{bmatrix} \gamma_{i,k} \\1 \end{bmatrix}\geq 0.
    \end{equation*}
Using the selector matrix $E_i$ defined in \eqref{eq:selectionmatrix}, we can rewrite the preceding inequality as
\begin{equation*}
\begin{bmatrix} \mathbf{z} \\1 \end{bmatrix}^\top  \underbrace{E_i^\top  \begin{bmatrix}-\Sigma_{\gamma_{i,k}}^{-1} & \Sigma_{\gamma_{i,k}}^{-1}\mu_{\gamma_{i,k}}\\ \mu_{\gamma_{i,k}}^\top \Sigma_{\gamma_{i,k}}^{-1} & -\mu_{\gamma_{i,k}}^\top  \Sigma_{\gamma_{i,k}}^{-1}\mu_{\gamma_{i,k}}+1  \end{bmatrix} E_i }_{M_i}\begin{bmatrix} \mathbf{z} \\1 \end{bmatrix}\geq 0,
\end{equation*}
This implies for any $\tau_i \geq 0$, the inequality \eqref{eq:InputQC} holds. (2) For the proof of the second part, we rewrite $\pi(z^0) \in \mathcal{E}(-U^{-1}V,\ U^{-2})$ as
\begin{equation}
    \begin{bmatrix}\pi(z^0) \\ 1 \end{bmatrix}^\top \begin{bmatrix}U^2 & UV \\ V^\top  U &V^\top  V -1\end{bmatrix}\begin{bmatrix}\pi(z^0) \\ 1 \end{bmatrix} \leq 0.
\end{equation}
This can be mapped to the space of vector $\mathbf{z}$ as,
\addtolength{\arraycolsep}{-3pt}
\begin{equation*}
    \begin{bmatrix}\mathbf{z} \\ 1 \end{bmatrix}^\top \begin{bmatrix}W^{\ell}S^{\ell} & b^\ell\\ 0 & 1\end{bmatrix}^\top  \begin{bmatrix}U^2 & UV \\ V^\top  U &V^\top  V \!-\!1\end{bmatrix} \begin{bmatrix}W^{\ell}S^{\ell} & b^\ell\\ 0 & 1\end{bmatrix}\begin{bmatrix}\mathbf{z} \\ 1 \end{bmatrix} \leq 0,
\end{equation*}
\addtolength{\arraycolsep}{3pt}
which characterizes $M_{out}$ provided in the lemma.
\end{proof}

The quadratic constraint for the activation layers, denoted by $\mathcal{Q}$ is presented in \cite{fazlyab2019safety} (Lemma 4) and remains unchanged in our formulation. Now consider the implicit equation $B \mathbf{z} = \phi(A \mathbf{z} + b)$ in \eqref{eq:NN2} describing the neural network, where $\phi$ satisfies the quadratic constraint defined by $\mathcal{Q}$. By Definition \ref{QQC} this implies that for any $Q \in \mathcal{Q}$, 
\begin{equation}\label{eq:activationQC}
    \begin{bmatrix}
    \mathbf{z} \\ 1
    \end{bmatrix}^\top \underbrace{\begin{bmatrix}A & b\\ B & 0\\ 0 & 1 \end{bmatrix}^\top  Q  \begin{bmatrix}A & b\\ B & 0\\ 0 & 1 \end{bmatrix}}_{M_{mid}(Q)} \begin{bmatrix}
    \mathbf{z} \\ 1
    \end{bmatrix} \geq 0.
\end{equation}
Returning to our main goal, bounding the output $\pi(z^0)$ by an ellipsoid given ellipsoidal bounds on the inputs $\gamma_{i,k}$, we know that the quadratic inequalities \eqref{eq:InputQC} and \eqref{eq:activationQC} hold and we would like to conclude the quadratic inequality \eqref{eq:OutputQC}. To this end we use the S-procedure.
%
%
%
Therefore, the multiple ellipsoidally bounded input version of Theorem 1 (Output covering ellipsoid) of \cite{fazlyab2019probabilistic} can be presented as follows.
%
\begin{theorem} \label{thm:multiellipsoid}
Consider the multi-layer neural network described by \eqref{eq:NN}. Suppose $\gamma_i \in \mathcal{E}(\mu_{\gamma_{i,k}}, \Sigma_{\gamma_{i,k}})$ and ReLU activation function $\phi$ satisfies the quadratic constraints defined by \eqref{eq:activationQC}. Let $U\in \mathbb{S}^{n_{\pi}}$, $V \in \mathbb{R}^{n_{\pi}}$ be two matrices that satisfy $M(\tau,Q,U,V) \preceq 0$ for some $Q \in \mathcal{Q}$ and $\tau \in \mathbb{R}_{+}^{q}$, where
\addtolength{\arraycolsep}{-3pt}
 \begin{equation}\label{eq:schurconstraint}
    M \!= \! \left[\begin{matrix}  
    \begin{array}{c|c}  
  \sum_{i=1}^{q} \tau_i M_i \! + \! M_{mid}(Q) \! - \! ee^\top & \begin{matrix}  
    0_{(N_{\mathbf{z}}-N_\ell) \times n_{\pi}} \\\\ W^{{\ell}^\top U} \\ \\ {b^{\ell}}^\top U+V^\top   
    \end{matrix} \\  
    \hline  
    \begin{matrix}  
    0_{n_{\pi} \times (N_{\mathbf{z}}-N_\ell)} \ \ \   UW^\ell \ \ \  Ub^\ell +V  
    \end{matrix}  & -I_{n_{\pi}}  
    \end{array}  
    \end{matrix}\right].
    \end{equation}
    \addtolength{\arraycolsep}{3pt}
%
Then we have
    $\pi(z^0) \in \mathcal{E}_{\pi(z^0)}=\mathcal{E}(-U^{-1}V, U^{-2})$.
\end{theorem}
\begin{proof}
According to the first part of Lemma \ref{lem:INOUTNN}, in order to guarantee $\pi(z^0) \in \mathcal{E}_{\pi(z^0)}=\mathcal{E}(-U^{-1}V, U^{-2})$, we need to provide a sufficient condition to satisfy the inequality,
\begin{equation}
    \begin{bmatrix}\mathbf{z} \\ 1 \end{bmatrix}^\top  M_{out} \begin{bmatrix}\mathbf{z} \\ 1 \end{bmatrix} \leq 0.
\end{equation}
Suppose the following matrix inequality holds,
\begin{equation} \label{eq:constraint}
    \sum_{i=1}^{q} \tau_i M_{i} + M_{mid}(Q) +M_{out} \preceq 0,
\end{equation}
for some $\tau_i \geq 0$, $Q \in \mathcal{Q}$. Based on \eqref{eq:constraint} we can conclude, 
    \begin{equation}
    \sum_{i=1}^{q} \underbrace{\begin{bmatrix}\mathbf{z} \\ 1 \end{bmatrix}^\top\!\!\!  \tau _i M_{i}\begin{bmatrix}\mathbf{z} \\ 1 \end{bmatrix}}_{\text{$\geq 0$}}+ \underbrace{\begin{bmatrix}\mathbf{z} \\ 1 \end{bmatrix}^\top\!\!\!   M_{mid} \begin{bmatrix}\mathbf{z} \\ 1 \end{bmatrix}}_{\text{$\geq 0 $}}+\begin{bmatrix}\mathbf{z} \\ 1 \end{bmatrix}^\top\!\!\!  M_{out} \begin{bmatrix}\mathbf{z} \\ 1 \end{bmatrix}\ \leq 0,
    \end{equation}
which provides what we need from $M_{out}$. The matrix inequality in \eqref{eq:constraint} is not linear in $U$ and $V$. However, it can be made linear through the Schur complement. First, rewrite $M_{out}$ as
\addtolength{\arraycolsep}{-2pt}
    \begin{align}
    M_{out}&=\begin{bmatrix}W^{\ell}S^{\ell} & b^\ell\\ 0 & 1\end{bmatrix}^\top  \begin{bmatrix}U & V \end{bmatrix}^\top  \begin{bmatrix}U & V \end{bmatrix} \begin{bmatrix}W^{\ell}S^{\ell} & b^\ell\\ 0 & 1\end{bmatrix} -  \nonumber\\
    &\begin{bmatrix}W^{\ell}S^{\ell} & b^\ell\\ 0 & 1\end{bmatrix}^\top  \begin{bmatrix}0_{1 \times n_{\pi}} & 1 \end{bmatrix}^\top  \begin{bmatrix}0_{1 \times n_{\pi}} & 1 \end{bmatrix} \begin{bmatrix}W^{\ell}S^{\ell} & b^\ell \nonumber\\ 0 & 1\end{bmatrix}\\
    &:=F^\top  F-ee^\top ,
    \end{align}
\addtolength{\arraycolsep}{2pt}
where
    \begin{equation}
    \begin{aligned}
    &F=\begin{bmatrix}0_{n_{\pi \times (N_{\mathbf{z}}-N_{\ell})}} &  UW^{\ell} & Ub^{\ell}+V\end{bmatrix},\\ &e=\underbrace{[\ 0,\ 0,\ 0,\ ..,\ 1]^\top }_\text{$N_{\mathbf{z}}+1$}.
    \end{aligned}
    \end{equation}
Through the application of the Schur complement on $M_{out}$, the matrix inequality in \eqref{eq:constraint} can be written as \eqref{eq:schurconstraint}, which is linear in $\left(V, U, Q \right)$.
\end{proof}   
Using the result of Theorem \ref{thm:multiellipsoid}, the tightest bound on $\pi(z^0)$ is obtained from the following semidefinite program,
\begin{equation}\label{eq:tightestbound}
    \left\{\begin{aligned}
    \mathrm{min}_{U,V,Q} &\  -\mathrm{logdet}(U)\  \text{or}\ -\mathrm{tr}(U)  \\
    \text{s.t.}\quad & M(\tau,Q,U,V) \preceq 0, \\
    & \tau \geq 0, \ Q \in \mathcal{Q} \text{ from \cite{fazlyab2019safety} (Lemma 4)},
    \end{aligned}\right.
\end{equation}    
where the objective function can be any metric related to the ellipsoid volume.    
\subsection{Application to Prediction Bound Propagation}
We now apply the general multiple ellipsoid input robustness bound developed in Theorem \ref{thm:multiellipsoid} to find the prediction bound $\mathcal{E}(\hat{\mu}_{k+1},\hat{\Sigma}_{k+1})$ with input ellipsoids $\mathcal{E}(y_{k-N},\bar{\Sigma}_v)$, $\cdots$, $\mathcal{E}(y_{k},\bar{\Sigma}_v)$. 
\begin{proposition}
Let $\gamma_{i,k} \in \mathcal{E}(y_{k-i+1},\bar{\Sigma}_v)$ for $i=1,2,\cdots,q=N+1$ and define $\hat{\Sigma}_{k+1}=U^{-2}$, $\hat{\mu}_{k+1}=-U^{-1}V$, and $\mathcal{C}_k$ as the actual (non-ellipsoidal) prediction set of the neural network. Then $\pi(z^0) \in \mathcal{C}_k$ is bounded by $\mathcal{E}(\hat{\mu}_{k+1}, \hat{\Sigma}_{k+1})$ where the tightest bound comes from the convex optimization \eqref{eq:tightestbound}.
\end{proposition}

\section{Numerical Experiments}
We consider a linear beam and slider and a nonlinear water tank cascade to demonstrate our proposed detector.

\begin{figure*}[ht]
    \begin{subfigure}{.28\textwidth}
    \centering \vspace{1.6cm}
    \includegraphics[width=0.95\linewidth]{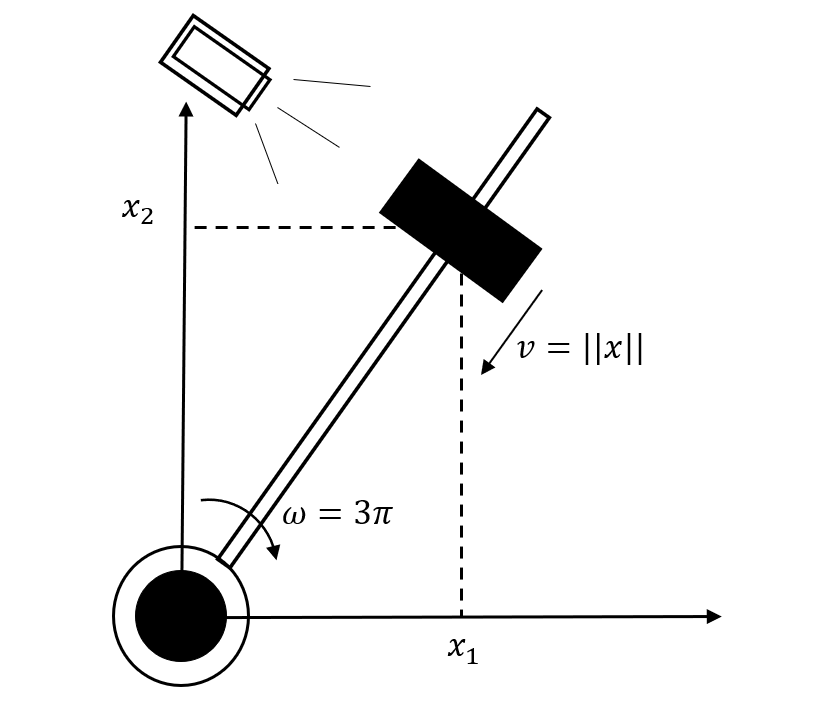}
    \caption{Beam and slider} \label{fig:beamandslider_system}
    \end{subfigure}
    \begin{subfigure}{.35\textwidth}
    \centering
    \includegraphics[width=0.95\linewidth]{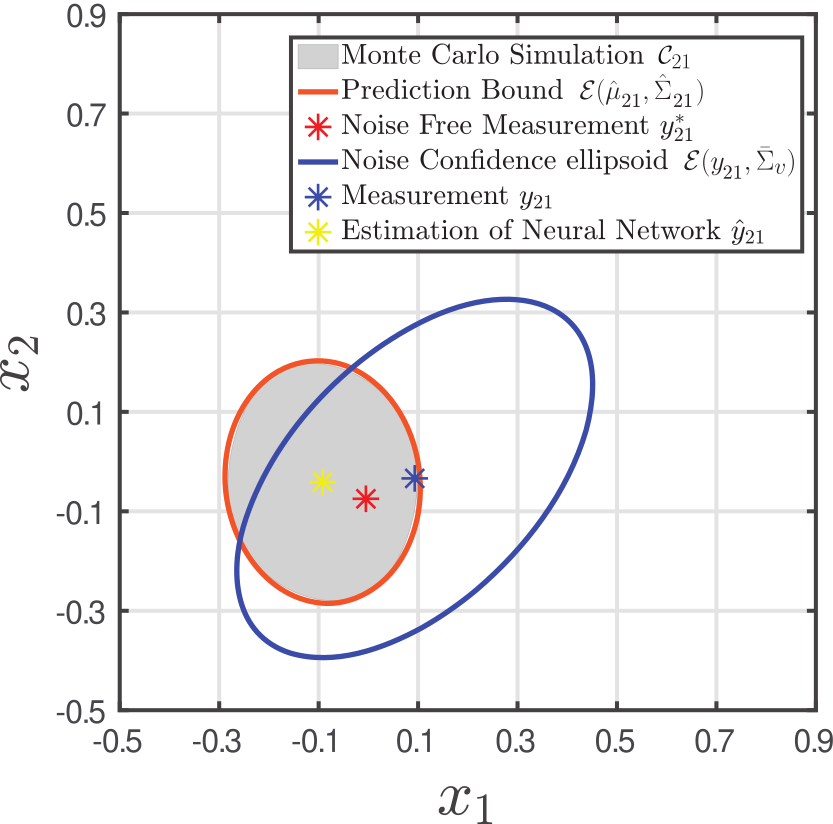}
    \caption{$k=20$ Prediction bound} \label{fig:beamandslider_prediction}
    \end{subfigure}
        \begin{subfigure}{.35\textwidth}
    \centering
    \includegraphics[width=0.95\linewidth]{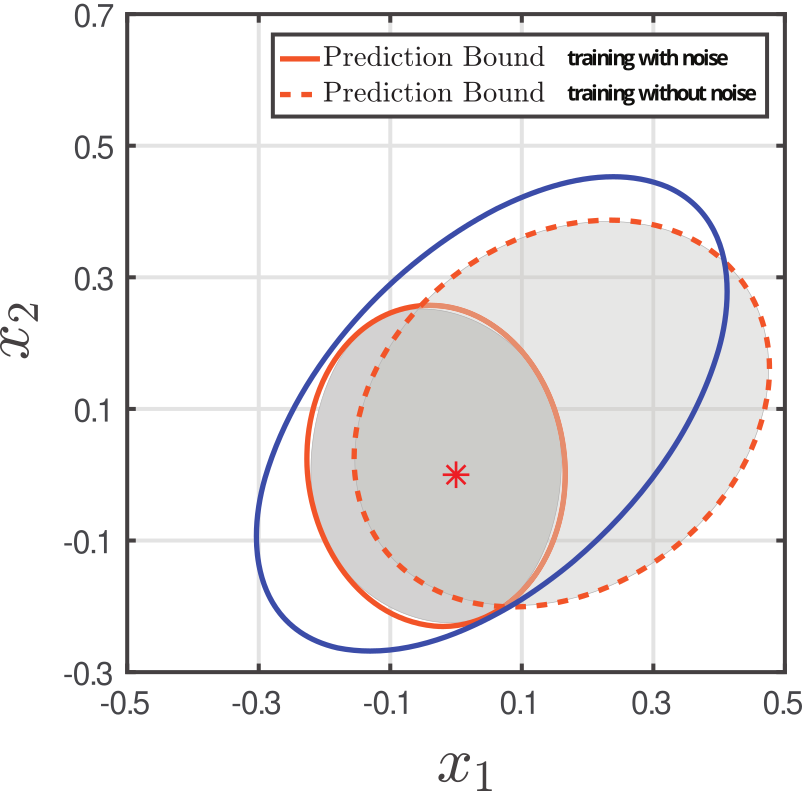}
    \caption{Robustness due to training} \label{fig:beamandslider_training}
    \end{subfigure}
    \caption{(a) The linear beam and slider system. (b) The effectiveness of the prediction bound at $k=20$. (c) A neural network trained with noisy data has better robustness characteristics than one trained with noiseless data and is quantified well by our prediction bounds.}
    \label{fig:beamandslider}
\end{figure*}

\subsection{Linear Beam and Slider}
In the beam and slider (see Fig. \ref{fig:beamandslider_system}), a sensor measures the position of the slider which slides along the beam with a rate equal to its distance from the origin. The beam rotates with constant angular velocity of $3\pi \frac{rad}{sec}$. The sensor noise is Gaussian distributed, $v_k \sim \mathcal{N}(0,\Sigma_v)$, and we truncate it with a $\bar{p}=95\%$ confidence ellipse $\mathcal{E}(y_{k},\bar{\Sigma}_v)$ with
\begin{equation*}
    \Sigma_v=\begin{bmatrix} 0.0214  &  0.0112 \\ 0.0112 &  0.0217  \end{bmatrix}\ \text{and}\  \bar{\Sigma}_v=\begin{bmatrix}0.1282 & 0.0671\\0.0671 & 0.1300\end{bmatrix}.
\end{equation*}

We train a NARX feedforward neural network with two hidden layers (10 neurons in the first layer, 2 neurons in the second) with ReLU activation functions from noisy sensor data. 
The data was generated through simulation of the beam and slider system based on the linear discrete time model,
\begin{equation}
\left\{\begin{aligned}
x_{k+1}&=0.8\begin{bmatrix*}[r]\text{cos}(\beta) & -\text{sin}(\beta) \\  \text{sin}(\beta) & \text{cos}(\beta)\end{bmatrix*}x_{k}, \quad \beta = \frac{3\pi}{5},\\
y_k&=x_k+v_k.
\end{aligned}\right.
\end{equation}
We use the current measurement and one past measurement (i.e., $N=1$) to build the NARX model (we still call it NARX because the neural network returns a nonlinear approximation of this linear model). 

Figure \ref{fig:beamandslider_prediction} demonstrates a snapshot of the efficacy of our methods under normal operation at $k=20$. The prediction bound $\mathcal{E}(\hat{x}_{21},\hat{\Sigma}_{21})$ (orange ellipse) can be quite tight on the actual prediction set $\mathcal{C}_{21}$ (gray region), which is the set of all possible neural network output estimates under all possible combinations from the input ellipsoids. As designed, the estimate $\hat{y}_{21}$ (is not necessarily the center $\hat{\mu}_{21}$) and ideal measurement $y^*_{21}$ (assuming good accuracy of the estimation) are within the prediction bound. The actual measurement $y_{21}$ is within the geometric sum of the prediction bound (orange) and confidence ellipsoid (blue); in this case it is also within the prediction bound, as the noise realization is relatively small. Under normal operation we received 0.7\% false alarms compared with the upper bound on the false alarm rate $\bar{\mathcal{A}}_s=1-\bar{p}^{3}=14.26\%$. This conservatism below the upper bound illustrates that the distribution within the prediction boundary, and the subsequent geometric sum, is an important factor for improving our ability to predict the false alarm rate and an important direction for future work.



Despite the conservatism, we demonstrate that it is an effective tool for detecting anomalies in behavior. The first fault we consider is a vibration generated on the shaft because of the rotor. This imposes a new additive periodic displacement on the shaft and consequently the slider in a fixed direction
\begin{equation}
 \delta x_k = \begin{bmatrix}\delta_1\\\delta_1\end{bmatrix} \text{sin}(k),
\end{equation}
where $\delta_1=0.3$ and $k$ is the time index. Under this fault the alarm rate raises to $27.17\%$.

The second scenario we consider is a sensor anomaly in which the sensor measurement is displaced by $\delta_2=0.3$,
\begin{equation}
y_k=x_k+v_k+\begin{bmatrix}\delta_2\\ \delta_2 \end{bmatrix}.
\end{equation}
Under this fault the alarm rate raises to $25.35\%$.

The prediction bound provides a quantification of the robustness of a neural network to perturbed input. It is intuitive that a neural network trained by noisy data should have enhanced robustness compared with a neural network trained only by ideal (no noise) sensor measurements since the model parameters are learned to filter the noise more effectively. The prediction bound allows us to characterize this intuitive relationship quantitatively. In Fig. \ref{fig:beamandslider_training}, we provide the prediction bounds for this system at exactly identical conditions, with the only change that one model is trained with the actual noisy data and the other uses the underlying ideal (no noise) outputs. The larger prediction ellipsoid bound for the neural network trained without noise (dashed ellipse) quantifies the value of training with noisy data.

\begin{figure}[t]
    \begin{subfigure}{.12\textwidth}
    \centering \vspace{0.5cm}
    \includegraphics[width=0.95\linewidth]{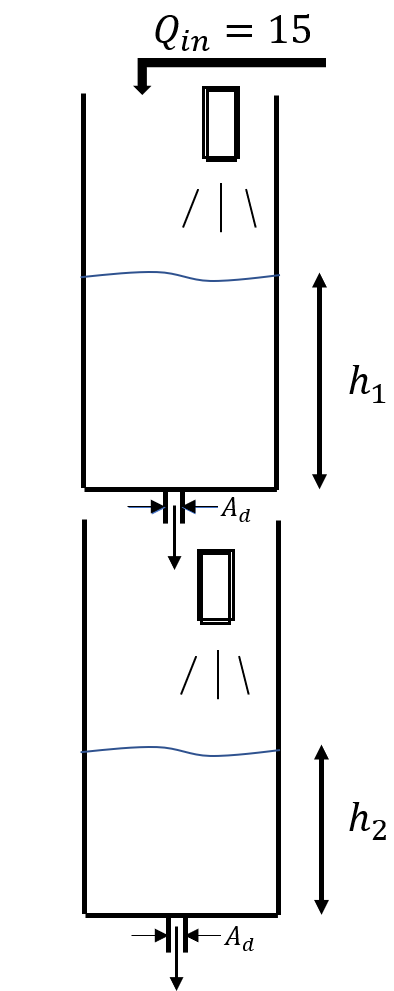}
    \vspace{0.01cm}
    \caption{Tank System} \label{fig:Twotank}
    \end{subfigure}
    \begin{subfigure}{.37\textwidth}
    \centering
    \includegraphics[width=0.95\linewidth]{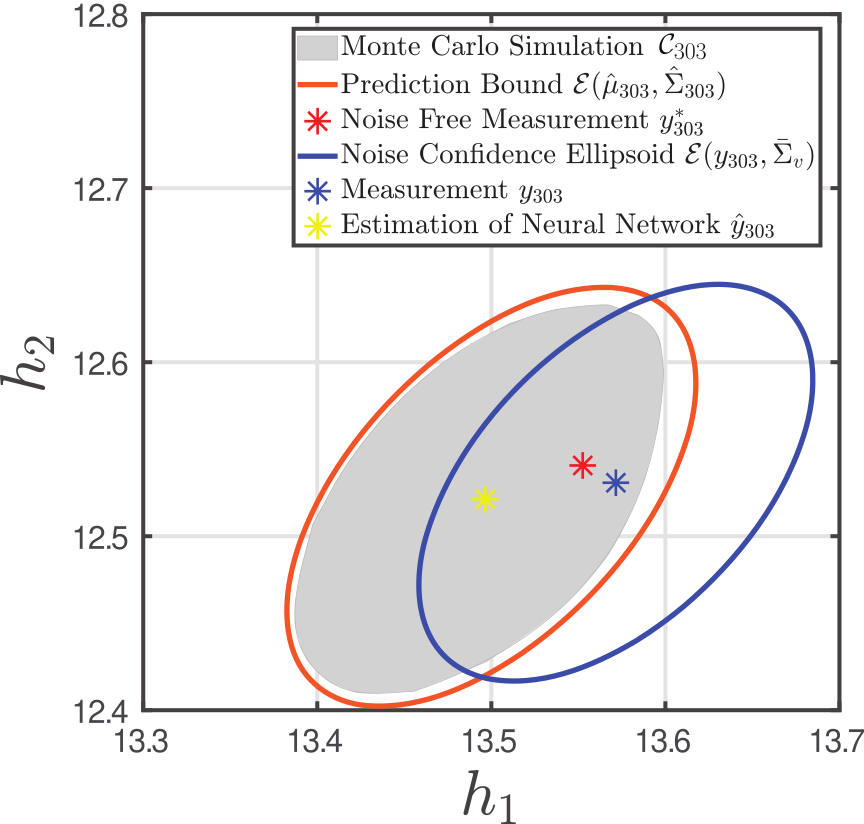}
    \caption{$k=302$ Prediction Bound} \label{fig:exampletank}
    \end{subfigure}
    \caption{(a) The nonlinear water tank system. (b) The effectiveness of the prediction bound at $k=302$.} 
    \label{fig:tankresults}
\end{figure}

\subsection{Nonlinear Water Tank Cascade}

We now consider a two tank cascade (see Fig. \ref{fig:Twotank}) in which flow out of the tanks is driven by gravity and yields a nonlinear dynamic,
\begin{equation}
\left\{\begin{aligned}
\dot{h_1}(t)&=Q_{in}-c_dA_d\sqrt{2gh_1(t)},\\
\dot{h_2}(t)&=c_dA_d\sqrt{2gh_1(t)}-c_dA_d\sqrt{2gh_2(t)},\\
y_k &= h(t_k) + v_k.
\end{aligned}\right.
\end{equation}
where $t_k=k\Delta t$ with inter-sampling period $\Delta t=0.02$ sec, $Q_{in}=15$, $c_d=0.9$, $A_d=1$ cm$^2$, and $g$ is the gravitational constant. For simplicity we use the same noise distribution and confidence level $\bar{p}$ as the prior example.

We train the neural network (two hidden layers: 20 neurons in the first layer; 5 neurons in the second layer) with the current noisy measurement and three (i.e., $N=3$) past noisy measurements generated by simulating the above equations. 

Figure \ref{fig:exampletank} provides a snapshot of the performance of our approach under normal operation, at $k=302$. Under normal operation the detector generates zero false alarms compared with the upper bound of $\bar{\mathcal{A}}_s=1-\bar{p}^{N+2}=22.62 \%$.


Here we simulate a fault in the drainage of the lower tank (e.g., debris caught in the drain) such that 20\% of the lower drain area is blocked. In this case, our proposed detector generates alarms at a rate of $58.1\%$ in steady state.

\section{Conclusion}
When first principle models of dynamical systems are too difficult to construct there is a growing tendency to turn to data-driven techniques. In this paper, we leverage recent results on the robustness of neural network predictions under input perturbations to compute bounds on the estimates produced from noisy measurements when a neural network is used to approximate an autoregressive model of the output measurements. We use this bound to define normal behavior under typical measurement noise and use this as a boundary to detect abnormal behavior. 

This detector performs well and provides tight detection for the extreme points of normal behavior. This study has revealed that despite the tightness on these extreme points, a more accurate picture of the distribution of normal behavior is needed to produce tighter alarm rates. Nonetheless, we demonstrate the proposed detector is able to identify anomalies of a variety of types in a systematic way.

\bibliographystyle{IEEEtran}
\bibliography{security}

\appendices
\section{comparison of multiple input approach with single input approach}\label{appendixA}
Consider the multiple ellipsoidal bounded inputs $\gamma_{i,k}$ introduced in Lemma \ref{lem:INOUTNN}. The quadratic constraint for this ellipsoid results to matrix $M_i$ and the sum $\sum_{i=1}^q \tau_i M_i$ can be written as (suppressing the $k$ index for readability)
\begin{equation}\label{eq:new}
    \begin{bmatrix}-\tau_1 \Sigma_{\gamma_{1}}^{-1}\!\! & 0&\cdots\!\! &\tau_1 \Sigma_{\gamma_{i}}^{-1}\mu_{\gamma_{1}}\\
    0&-\tau_2 \Sigma_{\gamma_{2}}^{-1}\!\! &\cdots\!\! &\tau_2\Sigma_{\gamma_{2}}^{-1}\mu_{\gamma_{2}}\\
    \vdots& \vdots & \ddots\!\! & \vdots\\
    \tau_1\mu_{\gamma_{1}}^\top  \Sigma_{\gamma_{1}}^{-1}\!\! &\tau_2 \mu_ {\gamma_{2}}^\top \Sigma_{\gamma_{2}}^{-1}\!\! & \dots\!\! &-\sum_{i=1}^q \tau_i\!\left(\mu_{\gamma_{i}}^\top  \Sigma_{\gamma_{i}}^{-1}\mu_{\gamma_{i}}\!\!-\!1\right)  \end{bmatrix}.
\end{equation}
On the other hand if we follow \cite{fazlyab2019probabilistic} and define an overall ellipsoidal bound on the input vector, this bound should satisfy,
\begin{equation}\label{eq:loosebound}
    \sum_{i=1}^{q} (\gamma_{i,k}-\mu_{\gamma_{i,k}})^\top  \Sigma_{\gamma_{i,k}}^{-1}(\gamma_{i,k}-\mu_{\gamma_{i,k}}) \leq q
\end{equation}
(compare to \eqref{eq:input_ellipsoidal_bound}) where the equality happens where all the points $\gamma_{i,k}$ are selected from the boundary of the input ellipsoidal bounds. Constructing $\sum_{i=1}^q \tau_i M_i$ in this case has only one term so $q=1$ and $\tau_1=\tau$ leading to
\begin{equation}\label{eq:old}
    \tau \begin{bmatrix}-\Sigma_{\gamma_{1}}^{-1} & 0&\cdots & \Sigma_{\gamma_{i}}^{-1}\mu_{\gamma_{1}}\\
    0&-\Sigma_{\gamma_{2}}^{-1} &\cdots &\Sigma_{\gamma_{2}}^{-1}\mu_{\gamma_{2}}\\
    \vdots& \vdots & \ddots & \vdots\\
    \mu_{\gamma_{1}}^\top  \Sigma_{\gamma_{1}}^{-1} & \mu_ {\gamma_{2}}^\top \Sigma_{\gamma_{2}}^{-1} & \dots &-\sum_{i=1}^q \mu_{\gamma_{i}}^\top  \Sigma_{\gamma_{i}}^{-1}\mu_{\gamma_{i}} +q \end{bmatrix}.
\end{equation}
Clearly \eqref{eq:old} is a specific form of \eqref{eq:new}, in which the single decision variable $\tau$ in \eqref{eq:old} is replaced with $q$ decision variables $\{\tau_i\}_{i=1}^q$ in \eqref{eq:new}. Hence any solution for \eqref{eq:old} can be expressed by $\eqref{eq:new}$ justifying that the multiple input ellipsoid approach can do no worse than the single input ellipsoid approach (i.e., it is as good or better than the single input approach).

\end{document}